\pdfoutput=1
\documentclass[letterpaper,english,5p,authoryear]{elsarticle}
\usepackage{color}
\usepackage{verbatim}
\usepackage{units}
\usepackage{mathrsfs}
\usepackage{url}
\usepackage{amsmath}
\usepackage{amsthm}
\usepackage{amssymb}
\usepackage{graphicx}

\makeatletter

\pdfpageheight\paperheight
\pdfpagewidth\paperwidth

\providecommand{\tabularnewline}{\\}

\theoremstyle{definition}
\newtheorem{defn}{\protect\definitionname}
\theoremstyle{plain}
\newtheorem{thm}{\protect\theoremname}
\theoremstyle{remark}
\newtheorem{rem}{\protect\remarkname}

\usepackage{ifthen}
\usepackage{url}	     

\usepackage{graphicx}
\usepackage{xcolor}
\usepackage{import}

\usepackage{caption}
\usepackage{microtype} 
\@ifundefined{showcaptionsetup}{}{%
 \PassOptionsToPackage{caption=false}{subfig}}
\usepackage{subfig}
\makeatother

\usepackage{babel}
\providecommand{\definitionname}{Definition}
\providecommand{\remarkname}{Remark}
\providecommand{\theoremname}{Theorem}

\begin{document}
\allowdisplaybreaks\global\long\def\dualquaternion#1{\underline{\boldsymbol{#1}}}%
\global\long\def\quaternion#1{\boldsymbol{#1}}%
\global\long\def\dq#1{\underline{\boldsymbol{#1}}}%
\global\long\def\quat#1{\boldsymbol{#1}}%
\global\long\def\mymatrix#1{\boldsymbol{#1}}%
\global\long\def\myvec#1{\boldsymbol{#1}}%
\global\long\def\mapvec#1{\boldsymbol{#1}}%
\global\long\def\dualvector#1{\underline{\boldsymbol{#1}}}%
\global\long\def\dual{\varepsilon}%
\global\long\def\dotproduct#1{\langle#1\rangle}%
\global\long\def\norm#1{\left\Vert #1\right\Vert }%
\global\long\def\mydual#1{\underline{#1}}%
\global\long\def\hamilton#1#2{\overset{#1}{\operatorname{\mymatrix H}}\left(#2\right)}%
\global\long\def\hamifour#1#2{\overset{#1}{\operatorname{\mymatrix H}}_{4}\left(#2\right)}%
\global\long\def\hami#1{\overset{#1}{\operatorname{\mymatrix H}}}%
\global\long\def\tplus{\dq{{\cal T}}}%
\global\long\def\getp#1{\operatorname{\mathcal{P}}\left(#1\right)}%
\global\long\def\getd#1{\operatorname{\mathcal{D}}\left(#1\right)}%
\global\long\def\swap#1{\text{swap}\{#1\}}%
\global\long\def\imi{\hat{\imath}}%
\global\long\def\imj{\hat{\jmath}}%
\global\long\def\imk{\hat{k}}%
\global\long\def\real#1{\operatorname{\mathrm{Re}}\left(#1\right)}%
\global\long\def\imag#1{\operatorname{\mathrm{Im}}\left(#1\right)}%
\global\long\def\imvec{\boldsymbol{\imath_{m}}}%
\global\long\def\vector{\operatorname{vec}_{6}}%
\global\long\def\mathpzc#1{\fontmathpzc{#1}}%
\global\long\def\cost#1#2{\underset{\text{#2}}{\operatorname{\text{cost}}}\left(\ensuremath{#1}\right)}%
\global\long\def\closedballset{\ensuremath{\mathbb{B}}}%

\global\long\def\complexset{\ensuremath{\mathbb{C}}}%

\global\long\def\realset{\ensuremath{\mathbb{R}}}%

\global\long\def\rationalset{\ensuremath{\mathbb{Q}}}%

\global\long\def\integerset{\ensuremath{\mathbb{Z}}}%

\global\long\def\naturalset{\ensuremath{\mathbb{N}}}%

\global\long\def\SE#1{\ensuremath{SE(#1)}}%

\global\long\def\SO#1{\ensuremath{SO(#1)}}%

\global\long\def\SU#1{\ensuremath{SU(#1)}}%

\global\long\def\liealgebraSE#1{\ensuremath{\mathfrak{se}\left(#1\right)}}%

\global\long\def\liealgebraSO#1{\ensuremath{\mathfrak{so}\left(#1\right)}}%

\global\long\def\liealgebraSU#1{\ensuremath{\mathfrak{su}\left(#1\right)}}%

\global\long\def\dualmatrix#1{\left\llbracket \left\llbracket #1\right\rrbracket \right\rrbracket }%
\global\long\def\quatset{\ensuremath{\mathbb{H}}}%

\global\long\def\purequatset{\mathbb{H}_{p}}%

\global\long\def\unitquatgroup{\mathrm{Spin}\left(3\right)}%

\global\long\def\unitquatset{\mathbb{S}^{3}}%

\global\long\def\dualquatset{\mathcal{H}}%
{} %

\global\long\def\unitdualquatgroup{\mathrm{Spin}\left(3\right)\ltimes\mathbb{R}^{3}}%
\global\long\def\unitdualquatset{\dq{\mathcal{S}}}%

\global\long\def\junta{q}%

\global\long\def\twist{\dq{\xi}}%

\global\long\def\diag#1{\operatorname{diag}\left(#1\right)}%

\global\long\def\diagterm#1{\mathfrak{D}\negmedspace\left\{  \negthinspace#1\negthinspace\right\}  }%

\global\long\def\jacobd{\mymatrix J}%

\global\long\def\jacobde{\mymatrix J}%

\global\long\def\dqxma{\dq x_{N}}%
\global\long\def\dqxmader{\dot{\dq x}_{N}}%

\global\long\def\error#1{\tilde{#1}}%

\global\long\def\errorvardq{\tilde{\dq z}}%
\global\long\def\errorvarquat{\tilde{\quat z}}%

\global\long\def\dqxmb{\tilde{\dq x}}%
\global\long\def\dqxmbder{\dot{\tilde{\dq x}}}%

\global\long\def\dqxmc{\dq{\tilde{x}}}%
\global\long\def\dqxmcder{\dot{\dq{\tilde{x}}}}%

\global\long\def\dqxmfinal{\dq x}%
\global\long\def\dqxmfinalder{\dot{\dq x}}%

\global\long\def\singulardist#1{\ifthenelse{\equal{#1}{f}}{f_{\sigma}}{\ifthenelse{\equal{#1}{max}}{\sigma_{\mathrm{far}}}{\ifthenelse{\equal{#1}{size}}{\sigma_{\mathrm{region}}}{\dq v_{s}}}}}%

\global\long\def\singmin#1{\sigma_{m}\negmedspace\left\{  \negthinspace#1\negthinspace\right\}  }%

\global\long\def\singvector#1{\myvec u_{m}\negmedspace\left\{  \negthinspace#1\negthinspace\right\}  }%

\global\long\def\puredualquatset{\mathcal{H}_{p}}%

\global\long\def\inner#1#2{\langle#1,#2\rangle}%

\global\long\def\vectorinv{\underline{\operatorname{vec}}_{6}}%

\global\long\def\teorientation#1{\operatorname{\mathcal{\mathscr{O}}\negmedspace}\left(#1\right) }%

\global\long\def\tetranslation#1{\operatorname{\mathscr{T}\negmedspace}\left(#1\right) }%

\global\long\def\perfvariable{{\gamma} }%

\global\long\def\perfrotation{{\gamma_{_{\mathcal{O}}}} }%

\global\long\def\perftranslation{{\gamma_{_{\mathcal{T}}}}}%

\global\long\def\KT{\kappa_{{\text{\tiny\ensuremath{\mathcal{T}}}}}}%

\global\long\def\KO{\kappa_{{\text{\tiny\ensuremath{\mathcal{O}}}}}}%

\global\long\def\vectorquat{\operatorname{vec}_{3}}%

\global\long\def\vectorquatinv{\underline{\operatorname{vec}}_{3}}%

\global\long\def\kinerotation{\quat r}%
\global\long\def\kineposition{\quat p}%

\global\long\def\errorrotation{\tilde{\quat r}}%
\global\long\def\errorposition{\tilde{\quat p}}%

\global\long\def\noiseaddvariablea{v}%
\global\long\def\noiseaddvariableb{w}%
\global\long\def\noiseadd{\dq{\noiseaddvariablea}_{\noiseaddvariableb}}%
\global\long\def\noiseaddposition{\quat{\noiseaddvariablea}_{\noiseaddvariableb}'}%
\global\long\def\noiseaddrotation{\quat{\noiseaddvariablea}_{\noiseaddvariableb}}%

\global\long\def\noiseposevariablea{v}%
\global\long\def\noiseposevariableb{c}%
\global\long\def\noisepose{\dq{\noiseaddvariablea}_{\noiseposevariableb}}%
\global\long\def\noiseposeeff{\bar{\dq{\noiseaddvariablea}}_{\noiseposevariableb}}%
\global\long\def\noiseposeposition{\quat{\noiseaddvariablea}_{\noiseposevariableb}'}%
\global\long\def\noiseposerotation{\quat{\noiseaddvariablea}_{\noiseposevariableb}}%

\global\long\def\noisedesiredvariablea{v}%
\global\long\def\noisedesiredvariableb{d}%

\global\long\def\noisedesired{\dq{\noisedesiredvariablea}_{\noisedesiredvariableb}}%
\global\long\def\noisedesiredposition{\quat{\noisedesiredvariablea}_{\noisedesiredvariableb}'}%
\global\long\def\noisedesiredrotation{\quat{\noisedesiredvariablea}_{\noisedesiredvariableb}}%
\global\long\def\noisedesiredbar{\bar{\dq{\noisedesiredvariablea}}_{\noisedesiredvariableb}}%

\begin{frontmatter}{}

\title{Robust $H_{\infty}$ kinematic control of manipulator robots using
dual quaternion algebra}

\author[tum]{Luis F. C. Figueredo\corref{cor1}}

\author[uom]{Bruno V. Adorno\corref{cor1}\corref{cor2}}

\author[unb]{Jo\~{a}o Y. Ishihara}

\cortext[cor1]{The first two authors contributed equally to this work.}

\cortext[cor2]{Corresponding author. \emph{Email:} bruno.adorno@manchester.ac.uk}

\address[tum]{Munich School of Robotics and Machine Intelligence, Technische Universit\"{a}t
M\"{u}nchen (TUM), 80797 Munich, Germany }

\address[uom]{Department of Electrical and Electronic Engineering, School of Engineering,
The University of Manchester, Sackville Street Manchester, M13 9PL
United Kingdom }

\address[unb]{Department of Electrical Engineering, University of Bras\'{i}lia (UnB)
\textendash{} 70910-900, Bras\'{i}lia, DF, Brazil}
\begin{abstract}
This paper proposes a robust dual-quaternion based $H_{\infty}$ task-space
kinematic controller for robot manipulators. To address the manipulator
liability to modeling errors, uncertainties, exogenous disturbances,
and their influence upon the kinematics of the end-effector pose,
we adapt $H_{\infty}$ techniques\textemdash suitable only for additive
noises\textemdash to unit dual quaternions. The noise to error attenuation
within the $H_{\infty}$ framework has the additional advantage of
casting aside requirements concerning noise distributions, which are
significantly hard to characterize within the group of rigid body
transformations. Using dual quaternion algebra, we provide a connection
between performance effects over the end-effector trajectory and different
sources of uncertainties and disturbances while satisfying attenuation
requirements with minimum instantaneous control effort. The result
is an easy-to-implement closed form $H_{\infty}$ control design criterion.
The performance of the proposed strategy is evaluated within different
realistic simulated scenarios and validated through real experiments.
\end{abstract}

\end{frontmatter}{}

\section{Introduction}

To ensure adequate performance, robot task-space kinematic controllers
must ensure robustness against modeling errors, uncertainties, and
exogenous disturbances that affect the end-effector pose. To cope
with the challenges that arise from the pose description and possible
representation singularities, the coupled translation and rotation
kinematics can be modeled using non-minimal representations such as
homogeneous transformation matrices (HTM) and unit dual quaternions.
The unit dual quaternion is a non-singular representation for rigid
transformations that is more compact, efficient and less computationally
demanding than HTM \citep{1998_Aspragathos_Dimitros_TSMC}. In addition,
dual quaternion algebra can represent rigid motions, twists, wrenches
and several geometrical primitives in a straightforward way, which
is useful when describing geometrical tasks directly in the task-space
\citep{Marinho2018}. Moreover, control laws are defined directly
over a vector field, eliminating the need to extract additional parameters
or to design matrix-based controllers.

Thanks to those advantages, there has been an increasing interest
in the study of kinematic representation and control in dual quaternion
space. Those works comprise rigid motion stabilization, tracking,
and multiple body coordination \citep{2008_Han_Wei_Li_IJAC,2012_Wang_Yu_Lin_TR,2013_Wang_Yu_SCL,Mas2017},
and kinematic control of manipulators with single and multiple arms
and human-robot interaction \citep{2010_Adorno_IROS,2013_Figueredo_Adorno_Ishihara_Borges_ICRA,2015_Adorno_Bo_Fraisse_ROB}.

Despite the developments on robot control using dual quaternion algebra,
there is still a gap in existing literature concerning the influence
of control parameters, uncertainties, and disturbances over tracking
robustness and performance when the trajectory is represented by unit
dual quaternions. 

\subsection{\textcolor{black}{Statement of contributions}}

We propose a robust dual-quaternion based $H_{\infty}$ task-space
kinematic controller for manipulators. The new method directly connects
different sources of uncertainties and disturbances to their corresponding
performance effects over the end-effector trajectory in dual quaternion
space. The controller explicitly addresses the influence of such disturbances
over the end-effector pose, in the $H_{\infty}$ sense, which does
not require detailed knowledge about the statistical distribution
of disturbances. This is paramount as those distributions are significantly
hard to characterize within the group $\text{Spin}(3)\ltimes\mathbb{R}^{3}$
of unit dual quaternions (or even $\SE 3$). Using dual quaternion
algebra, we derive easy-to-implement closed form $H_{\infty}$ control
and tracking strategies at the end-effector level that incorporate
robustness requirements, disturbance attenuation and performance properties
over the pose kinematics, while minimizing the required control effort.
In summary, the contributions to the state of the art are:
\begin{enumerate}
\item Introduction of novel geometrical description of disturbances within
the space of unit dual quaternions;
\item Development of an easy-to-implement, closed form $H_{\infty}$ controller
for end-effector trajectory tracking.
\end{enumerate}

\section{Preliminaries\label{sec:Mathematical-Background}}


The algebra of quaternions is generated by the basis elements $1,\imi,\imj,$
and $\imk$ and a distributive multiplication operation satisfying
$\imi^{2}{=}\imj^{2}{=}\imk^{2}{=}\imi\imj\imk{=}-1$, yielding the
set 
\[
\quatset\triangleq\smash{\left\{ \eta+\mu_{1}\imi+\mu_{2}\imj+\mu_{3}\imk\,:\,\eta,\mu_{1},\mu_{2},\mu_{3}\in\realset\right\} }.
\]
 An element $\quat h=\eta+\mu_{1}\imi+\mu_{2}\imj+\mu_{3}\imk\in\mathbb{H}$
may be decomposed into real and imaginary components $\real{\quat h}\triangleq\eta$
and $\imag{\quat h}\triangleq\mu_{1}\imi+\mu_{2}\imj+\mu_{3}\imk$,
such that $\quat h=\real{\quat h}+\imag{\quat h}$.

Quaternion elements with real part equal to zero belong to the set
of pure quaternions $\mathbb{H}_{p}\triangleq\left\{ \quat h\in\quatset\ :\real{\quat h}\right.$
$\left.=0\right\} $, and are equivalent to vectors in $\realset^{3}$
under the addition operation and the bijective operator $\vectorquat:\quatset_{p}{\rightarrow}\mathbb{R}^{3}$,
such that $\myvec{\mu}{=}\mu_{1}\imi+\mu_{2}\imj+\mu_{3}\imk$ yields
$\vectorquat\quat{\mu}{=}\begin{bmatrix}\mu_{1} & \mu_{2} & \mu_{3}\end{bmatrix}^{T}$.
The inverse mapping is given by the operator $\vectorquatinv$.

The set of unit quaternions is defined as $\unitquatset\triangleq\{\quat h\in\quatset\,:\,\norm{\quat h}=1\}$,
where $\norm{\quat h}\triangleq\sqrt{\quat h\quat h^{\ast}}=\sqrt{\quat h^{\ast}\quat h}$
is the quaternion norm and $\quat h{}^{\ast}\triangleq\real{\quat h}-\imag{\quat h}$
is the conjugate of $\quat h$. The set $\unitquatset$, together
with the multiplication operation, forms the Lie group of unit quaternions,
$\unitquatgroup$, whose identity element is $1$ and the inverse
of any element $\quat h\in\unitquatgroup$ is $\quat h^{\ast}$. An
arbitrary rotation angle $\phi\in\realset$ around the rotation axis
$\quat n\in\quatset_{p}\cap\unitquatset$, with $\quat n=n_{x}\imi+n_{y}\imj+n_{z}\imk$,
is represented by $\quat r=\cos(\phi/2)+\sin(\phi/2)\quat n\in\unitquatgroup$
\citep{Selig2005}.

The complete rigid body displacement, in which translation and rotation
are coupled, is similarly described using dual quaternion algebra
\citep{Selig2005}. The dual quaternion set is given by the set 
\[
\dualquatset\triangleq\{\quat h+\dual\quat h'\,:\,\quat h,\quat h'\in\quatset,\ \dual^{2}=0,\ \dual\neq0\},
\]
where $\dual$ is the dual unit. Given $\dq h=\quat h+\dual\quat h'\in\mathcal{H}$,
its norm is defined as $\norm{\dq h}\triangleq\sqrt{\dq h\dq h^{*}}=\sqrt{\dq h^{*}\dq h}$
and the element $\dq h^{*}\triangleq\quat h^{*}+\dual\quat h'^{*}$
is the conjugate of $\dq h$. Under multiplication, the subset of
\emph{unit }dual quaternions $\unitdualquatset\triangleq\{\dq h\in\dualquatset\,:\,\norm{\dq h}=1\}$
forms the Lie group $\text{Spin}(3)\ltimes\mathbb{R}^{3}$, whose
identity element is $1$ and the group inverse of $\dq x\in\text{Spin}(3)\ltimes\mathbb{R}^{3}$
is $\dq x^{\ast}$ \citep{Selig2005}. An arbitrary rigid displacement
defined by a translation $\kineposition\in\quatset_{p}$ followed
by a rotation $\kinerotation\in\unitquatset$ is represented in $\text{Spin}(3)\ltimes\mathbb{R}^{3}$
by the element $\dq x=\kinerotation+(1/2)\dual\kineposition\kinerotation$.

The first order kinematic equation of a rigid body motion is described
by
\begin{equation}
\dot{\dq x}=\smash{\frac{1}{2}}\twist\dq x,\label{eq:def:DQ kinematics}
\end{equation}
where $\twist=\quat{\omega}+\dual\left(\dot{\kineposition}+\kineposition\times\quat{\omega}\right)$
is the twist in the inertial frame and $\quat{\omega},\dot{\kineposition}\in\mathbb{H}_{p}$
are the angular and linear velocities, respectively. The twist $\dq{\xi}$
belongs to the set of pure dual quaternions, defined as $\puredualquatset\triangleq\{\left(\quat h+\dual\quat h'\right)\in\dualquatset\,:\,\real{\quat h}{=}\real{\quat h'}{=}0\}$,
which is equivalent to vectors in $\mathbb{R}^{6}$ under the addition
operation and the bijective operator $\vector:\puredualquatset\rightarrow\mathbb{R}^{6}$,
such that $\twist=(\xi_{1}\imi+\xi_{2}\imj+\xi_{3}\imk)+\dual(\xi_{4}\imi+\xi_{5}\imj+\xi_{6}\imk)$
yields $\vector\twist=\begin{bmatrix}\xi_{1} & {\cdots} & \xi_{6}\end{bmatrix}^{T}$.
The inverse mapping is denoted by $\vectorinv:\realset^{6}\rightarrow\puredualquatset$.

\subsection{Forward Kinematics of Serial Manipulators\label{sec:Kinematics-Modeling-And}}

The rigid transformation from the robot's fixed base to its end-effector
pose\textemdash i.e., its forward kinematics\textemdash is described
by $\dqxma\left(\myvec q\right)=\dq x_{1}^{0}\dq x_{2}^{1}\dots\dq x_{n}^{n-1}$,
with $\myvec q{=}\begin{bmatrix}q_{1} & {\cdots} & q_{n}\end{bmatrix}^{T}$,
where $\dq x_{i+1}^{i}\triangleq\dq x_{i+1}^{i}\left(q_{i+1}\right)\in\unitdualquatgroup$
represents the rigid transformation between the extremities of links
$i$ and $i+1$ and is a function of joint configuration $\junta_{i+1}\in\mathbb{R}$.

The differential forward kinematics, which describes the mapping between
the joints velocities and the end-effector (generalized) velocity
is given by \citep{2011_Adorno_THESIS}
\begin{align}
\dqxmader & =\frac{1}{2}\dq{\xi}_{N}\dqxma=\frac{1}{2}{\sum}_{i=1}^{n}\dq{\jmath}_{i}\dot{\junta}_{i}\dqxma,\label{eq:system:DQ Diif_FKM:continuous}\\
\dq{\jmath}_{i} & =2\dq x_{i-1}^{0}\frac{d\dq x_{i}^{i-1}}{dq_{i}}\left(\dq x_{i}^{i-1}\right)^{*}\dq x_{0}^{i-1}.
\end{align}

\section{Uncertainties and Exogenous Disturbances}


In practice, the end-effector trajectory is likely to be influenced
by different sources of exogenous disturbances and inaccuracies in
the manipulator's geometrical parameters, resulting in an uncertain
differential forward kinematics. To improve accuracy and control performance,
the influence of those uncertainties and disturbances over the system
must be explicitly regarded. Thus, we investigate two sources of disturbances,
namely twist and pose uncertainties.

Twist uncertainties may be caused by exogenous disturbances that directly
influence the end-effector \emph{velocity}, such as unmodeled time-varying
uncertainties, aerodynamic forces acting on non-rigid manipulators,
and the effects of discrete implementations of controllers designed
in continuous time. The differential forward kinematics under the
influence of a twist disturbance $\dq v_{w}$ is modeled by\footnote{Notice that, since $\noiseadd{\in}\dualquatset_{p}$ is in the Lie
algebra of $\text{Spin}(3)\ltimes\mathbb{R}^{3}$, $\dqxmader$ belongs
to the tangent space of $\unitdualquatgroup$ at $\dq x_{N}$.}
\begin{equation}
\dqxmader=\frac{1}{2}{\sum}_{i=1}^{n}\dq{\jmath}_{i}\dot{\junta}_{i}\dqxma+\frac{1}{2}\noiseadd\dqxma.\label{eq:Diif_FKM:vd_vw}
\end{equation}

Pose uncertainties, which are related to the end-effector \emph{pose},
may also arise from unforeseen inaccuracies within model geometric
parameters and time-varying uncertainties, but also comprise inaccuracies
in the location of the reference frame. They can also appear due to
the unmodeled parameters of the dynamic model, e.g., the gravity effect
on the robot. Because they affect the forward kinematics, pose uncertainties
can be mapped to \eqref{eq:system:DQ Diif_FKM:continuous}\textendash \eqref{eq:Diif_FKM:vd_vw}
as
\begin{equation}
\dqxmfinal=\dqxma\dq c,\label{eq:system:vc_in_DQ}
\end{equation}
where $\dq c\in\unitdualquatgroup$ and $\dqxmfinal$ denotes the
real pose of the disturbed end-effector. The time derivative of \eqref{eq:system:vc_in_DQ}
yields
\begin{align*}
\dqxmfinalder & =\dqxmader\dq c+\dqxma\dot{\dq c}=\frac{1}{2}{\sum}_{i=1}^{n}\dq{\jmath}_{i}\dot{\junta}_{i}\dqxmfinal+\frac{1}{2}\noiseadd\dqxmfinal+\frac{1}{2}\dqxmfinal\noiseposeeff,
\end{align*}
where $\noiseposeeff\in\puredualquatset$ is the twist related to
$\dot{\dq c}$, but expressed in the local frame; that is, $\dot{\dq c}=\left(1/2\right)\dq c\noiseposeeff$.
Since the disturbance $\noiseposeeff$ can be expressed in the inertial
frame by means of the norm-preserving transformation $\noiseposeeff=\dqxmfinal^{*}\noisepose\dqxmfinal$,
the actual differential forward kinematics, under twist and pose uncertainties,
is described by
\begin{equation}
\dqxmfinalder=\smash{\frac{1}{2}}{\sum}_{i=1}^{n}\dq{\jmath}_{i}\dot{\junta}_{i}\dqxmfinal+\smash{\frac{1}{2}}\noiseadd\dqxmfinal+\smash{\frac{1}{2}}\noisepose\dqxmfinal.\label{eq:system:final differential kinematics}
\end{equation}

\subsection{Tracking Error Definition}

Given a desired differentiable pose trajectory $\dq x_{d}\left(t\right)\in\unitdualquatgroup$,
we seek to guarantee internal stability and tracking performance in
terms of the noise-to-output influence over the end-effector trajectory.
From \eqref{eq:def:DQ kinematics}, $\dq x_{d}\left(t\right)$ satisfies
the first order kinematic equation
\begin{equation}
\dot{\dq x}_{d}=\smash{\frac{1}{2}}\twist_{d}\dq x_{d}.\label{eq:def:DESIRED DQ Kinematics}
\end{equation}
We define the spatial difference in $\text{Spin}(3)\ltimes\mathbb{R}^{3}$
as
\begin{equation}
\error{\dq x}\triangleq\dqxmfinal\dq x_{d}^{\ast}=\error{\quat r}+\dual\smash{\frac{1}{2}}\error{\quat p}\error{\quat r},\label{eq:error:NonDisturbed:spatial error}
\end{equation}
where $\error{\quat r}=\quat r\quat r_{d}^{\ast}$ denotes the orientation
error in $\unitquatgroup$ given the desired orientation $\quat r_{d}$,
and $\error{\quat p}=\quat p-\error{\quat r}\quat p_{d}\error{\quat r}^{\ast}$
denotes the translational error in $\purequatset$ given the desired
position $\quat p_{d}$.

From \eqref{eq:system:final differential kinematics} and \eqref{eq:def:DESIRED DQ Kinematics},
the error kinematics is given by
\begin{align}
\dot{\error{\dq x}} & =\dot{\dq x}\dqxmfinal{}_{d}^{\ast}+\dqxmfinal\dot{\dq x}_{d}^{\ast}=\frac{1}{2}\left(\sum_{i=1}^{n}\dq{\jmath}_{i}\dot{\junta}_{i}+\noiseadd+\noisepose\right)\error{\dq x}-\frac{1}{2}\error{\dq x}\twist_{d}\nonumber \\
 & =\frac{1}{2}\left(\vectorinv\left(\mymatrix J\dot{\myvec{\junta}}\right)+\noiseadd+\noisepose\right)\error{\dq x}-\frac{1}{2}\error{\dq x}\twist_{d},\label{eq:tracking:error_dynamics:jacobian}
\end{align}
where $\myvec{\junta}=\smash{\begin{bmatrix}\junta_{1} & \cdots & \junta_{n}\end{bmatrix}^{T}}$
is the measured vector of joint variables and $\mymatrix J=\begin{bmatrix}\vector\dq{\jmath}_{1} & \cdots & \vector\dq{\jmath}_{n}\end{bmatrix}$
is the analytical Jacobian that maps the joints velocities $\dot{\myvec q}$
to the (undisturbed) twist $\vector\dq{\xi}_{N}$ of the end-effector.

From the spatial difference \eqref{eq:error:NonDisturbed:spatial error},
we define a right invariant dual quaternion error function\footnote{In order to prevent the unwinding phenomenon, see Remark~\ref{rem:addressing_the_unwinding}.}
\begin{equation}
\errorvardq\triangleq1-\error{\dq x}=\errorvarquat+\dual\errorvarquat'\label{eq:error:NonDisturbed:new metrics}
\end{equation}
with dynamics described by $\dot{\errorvardq}=-\dot{\error{\dq x}}$.
Therefore, $\errorvardq\rightarrow0$ implies $\error{\dq x}\rightarrow1$,
which implies $\dqxmfinal\rightarrow\dq x_{d}$. 

To address the detrimental influence of the uncertainties and disturbances
in system \eqref{eq:tracking:error_dynamics:jacobian}, we address
as variable of interest the orientation and position errors from \eqref{eq:error:NonDisturbed:spatial error}
and \eqref{eq:error:NonDisturbed:new metrics}, defined respectively
as
\begin{align}
\teorientation{\errorvardq} & \triangleq\imag{\errorvarquat}, & \tetranslation{\errorvardq} & \triangleq-2\errorvarquat'(1{-}\errorvarquat^{\ast})=\error{\kineposition}.\label{eq:error_at_origin}
\end{align}

\subsection{Performance Under Uncertainties and Disturbances}

The tracking error defined in the previous subsection explicitly accounts
for uncertainties and noises in the closed-loop control of the robotic
arm, allowing a performance assessment for any control strategy. If
the statistics of the uncertainties and noises are available, a stochastic
analysis can be considered \citep{2006_Simon_BOOK}. However, for
non-Euclidean spaces the probability density functions are, in general,
hard to characterize and, when available, difficult to manipulate.
In this paper, we propose a deterministic performance analysis based
on the $H_{\infty}$ approach \citep{2006_AbuKhalaf_Huang_Lewis_BOOK},
in which the effect of the input onto the output is intuitively measured
as a maximal level of amplification. The main advantage is the needlessness
for assumptions regarding the statistics of the uncertainties and
noises. As a result, the analysis is simpler than the stochastic one.

The following definition describes the robust performance (in the
$H_{\infty}$ sense) in terms of the dual quaternion error \eqref{eq:error:NonDisturbed:new metrics}
and the disturbances $\noiseadd=\noiseaddrotation+\dual\noiseaddposition$
and $\noisepose=\noiseposerotation+\dual\noiseposeposition$, assuming
$\noiseaddrotation,\noiseaddposition,\noiseposerotation,\noiseposeposition\in L_{2}([0,\infty),\purequatset)$.\footnote{$L_{2}$ is the Hilbert space of all square-integrable functions.}
\begin{defn}
\label{def:H_inf}For $\perfrotation_{1},\perfrotation_{2},\perftranslation_{1},\perftranslation_{2}\in\left(0,\infty\right)$,
the robust control performance is achieved, in the $H_{\infty}$ sense,
if the following hold \citep{2006_AbuKhalaf_Huang_Lewis_BOOK}

(1) The error \eqref{eq:error:NonDisturbed:new metrics} is exponentially
stable for $\noiseadd{=}\noisepose{=}0$;

(2) Under the assumption of zero initial conditions, the disturbances'
influence upon the attitude and translation errors is attenuated below
a desired level; that is, $\forall\left(\noiseaddrotation,\noiseposerotation,\noiseaddposition,\noiseposeposition\right)\in L_{2}((0,\infty),\purequatset)$
\begin{align*}
\begin{aligned}\int_{0}^{\infty}\negthickspace\negthickspace\negthickspace\norm{\teorientation{\errorvardq\left(t\right)}}^{2}\negthickspace dt & \leq\perfrotation_{1}^{\negthickspace\negthickspace2}\int_{0}^{\infty}\negthickspace\negthickspace\negthickspace\norm{\noiseaddrotation\left(t\right)}^{2}\negthickspace dt+\perfrotation_{2}^{\negthickspace\negthickspace2}\int_{0}^{\infty}\negthickspace\negthickspace\negthickspace\norm{\noiseposerotation\left(t\right)}^{2}\negthickspace dt,\\
\int_{0}^{\infty}\negthickspace\negthickspace\negthickspace\norm{\tetranslation{\errorvardq\left(t\right)}}^{2}\negthickspace dt & \leq\perftranslation_{1}^{\negthickspace\negthickspace2}\int_{0}^{\infty}\negthickspace\negthickspace\negthickspace\norm{\noiseaddposition\left(t\right)}^{2}\negthickspace dt+\perftranslation_{2}^{\negthickspace\negthickspace2}\int_{0}^{\infty}\negthickspace\negthickspace\negthickspace\norm{\noiseposeposition\left(t\right)}^{2}\negthickspace dt.
\end{aligned}
\end{align*}
\end{defn}
The $H_{\infty}$ criterion determines the maximum ratio of the error
to the disturbance, in terms of their $L_{2}$-norms, such that the
parameters $\perfrotation_{1}^{\negthickspace\negthickspace2},\perfrotation_{2}^{\negthickspace\negthickspace2},\perftranslation_{1}^{\negthickspace\negthickspace2},\perftranslation_{2}^{\negthickspace\negthickspace2}$
refer to the upper bounds of the performance index of each separate
disturbance effect.

\section{$H_{\infty}$ Control Strategies\label{sec:-CONTROL-STRATEGIES}}


In this section, we exploit the dual quaternion algebra to solve the
$H_{\infty}$ kinematic control problem while accounting for both
additive and multiplicative disturbances.
\begin{thm}[$H_{\infty}$ Tracking Control\footnote{Set-point control (i.e., regulation) is a particular case that can
be achieved by letting $\dq{\xi}_{d}=0$ in \eqref{THM:Tracking Controller}.}]
\label{THEOREM_TRACKING}Let $\mymatrix J^{+}$ be the Moore-Penrose
pseudo-inverse of $\mymatrix J$, and $\teorientation{\errorvardq}$
and $\tetranslation{\errorvardq}$ be given by \eqref{eq:error_at_origin}.
For $\perfrotation_{1},\perfrotation_{2},\perftranslation_{1},\perftranslation_{2}\in\left(0,\infty\right)$,
the task-space kinematic controller yielding joints' velocity inputs
\begin{align}
\dot{\myvec{\junta}} & =\mymatrix J^{+}\left(\begin{bmatrix}\KO\vectorquat\teorientation{\errorvardq}\\
-\KT\vectorquat\tetranslation{\errorvardq}
\end{bmatrix}+\vector\left(\error{\dq x}\twist_{d}\error{\dq x}^{\ast}\right)\right),\label{THM:Tracking Controller}
\end{align}
where $\KO=\left(\perfrotation_{1}^{\negthickspace\negmedspace-2}+\perfrotation_{2}^{\negthickspace\negmedspace-2}\right)^{1/2}$
and $\KT=\left(\perftranslation_{1}^{\negthickspace\negmedspace-2}+\perftranslation_{2}^{\negthickspace\negmedspace-2}\right)^{1/2}$,
ensures exponential $H_{\infty}$ tracking performance with disturbance
attenuation in the sense of Definition~\ref{def:H_inf}. Furthermore,
if $\gamma{\triangleq}\perftranslation_{1}{=}\perftranslation_{2}{=}\perfrotation_{1}{=}\perfrotation_{2}$
such that $\KO=\KT=\sqrt{2}\gamma^{-1}$, then the aforementioned
gains $\KO$ and $\KT$ ensure the minimum instantaneous control effort
(i.e., minimum norm of the control inputs) for the closed-loop system
\eqref{eq:tracking:error_dynamics:jacobian},\eqref{THM:Tracking Controller}.
\end{thm}
\begin{proof}
First, we replace \eqref{THM:Tracking Controller} in \eqref{eq:tracking:error_dynamics:jacobian}
to obtain\footnote{Eq.~\eqref{eq:equivalent tracking controller} holds even if $\twist_{d}$
is not feasible, i.e., $\mymatrix J\mymatrix J^{+}\vector(\error{\dq x}\twist_{d}\error{\dq x}^{\ast})\neq\vector(\error{\dq x}\twist_{d}\error{\dq x}^{\ast})$.
Let $\myvec s\triangleq\vector(\error{\dq x}\twist_{d}\error{\dq x}^{\ast})$
then $\vectorinv(\mymatrix J\mymatrix J^{+}\myvec s)=\vectorinv\left(\myvec s\right)+\dq v_{s}$,
where $\dq v_{s}$ is just another source of disturbance to be added
into $\noiseadd$.\label{fn:if right-pseudoinverse-not-exist-result-in-disturbance}} 
\begin{align}
\dot{\error{\dq x}} & =\smash{\frac{1}{2}}\left(\vectorinv\left(\mymatrix J\dot{\bar{\myvec{\junta}}}\right)+\noiseadd+\noisepose\right)\error{\dq x},\label{eq:equivalent tracking controller}
\end{align}
where $\dot{\bar{\myvec{\junta}}}=\mymatrix J^{+}\begin{bmatrix}\KO\left(\vectorquat\teorientation{\errorvardq}\right)^{T} & -\KT\left(\vectorquat\tetranslation{\errorvardq}\right)^{T}\end{bmatrix}^{\smash{T}}$.

\noindent \textbf{(Exponential stability)} To study the stability
of the closed loop system, let us regard the following Lyapunov candidate
function
\begin{equation}
\begin{aligned}V(\errorvardq\left(t\right)) & =V_{1}(\errorvarquat\left(t\right))+V_{2}(\errorvarquat'\left(t\right)),\end{aligned}
\label{THM:LYAPUNOV}
\end{equation}
where $V_{1}(\errorvarquat\left(t\right))\triangleq\alpha_{1}\norm{\errorvarquat(t)}^{2}$
and $V_{2}(\errorvarquat'\left(t\right))\triangleq\alpha_{2}\norm{\errorvarquat'(t)}^{2}$
with $\alpha_{1},\alpha_{2}\in\left(0,\infty\right)$. The time-derivative
of \eqref{THM:LYAPUNOV}, considering \eqref{eq:equivalent tracking controller}
(see \ref{APPENDIX_A}) in the absence of disturbances (i.e., $\noiseadd=\noisepose=0$)
yields
\begin{align*}
\dot{V}_{1}(\errorvarquat) & \leq-\frac{\KO}{2}\alpha_{1}\norm{\errorvarquat}^{2}, & \dot{V}_{2}(\errorvarquat') & \!=\!-2\KT\alpha_{2}\norm{\errorvarquat'}^{2}.
\end{align*}
Hence, the closed-loop system, in the absence of disturbances, satisfy
the following inequalities 
\begin{align}
\dot{V}(\errorvardq\left(t\right)) & \leq-\frac{\KO}{2}\alpha_{1}\norm{\errorvarquat(t)}^{2}-2\KT\alpha_{2}\norm{\errorvarquat'(t)}^{2}\label{eq:proof:Lyapunov Derivative 1}\\
 & \leq-\min\left\{ \frac{\KO}{2},2\KT\right\} V(\errorvarquat\left(t\right))\leq0,\nonumber 
\end{align}
which implies, by the Comparison Lemma \citep[p. 85]{1996_Khalil_BOOK},
that the closed-loop system is exponentially stable: 
\begin{align*}
V(\errorvardq\left(t\right)) & \leq V(\errorvarquat\left(t_{0}\right))\exp\left(-\min\left\{ \frac{\KO}{2},2\KT\right\} \left(t-t_{0}\right)\right).
\end{align*}
This way, Condition~1 in Definition~\ref{def:H_inf} is satisfied
for $\KO,\KT\in\left(0,\infty\right)$. In addition, by using the
Comparison Lemma together with \eqref{eq:Derivative of V2} and \eqref{eq:Derivative of V1 (inequality)},
both individual attitude and translation dynamics achieve exponential
stability in the absence of disturbances, that is, 
\begin{align*}
\norm{\errorvarquat(t)}^{2} & \leq\norm{\errorvarquat(t_{0})}^{2}\exp\left(-\frac{1}{2}\KO\left(t-t_{0}\right)\right),\\
\norm{\errorvarquat'(t)}^{2} & =\norm{\errorvarquat'(t_{0})}^{2}\exp\left(-2\KT\left(t-t_{0}\right)\right).
\end{align*}

\noindent \textbf{(Disturbance attenuation)} To verify Condition~2
in Definition~\ref{def:H_inf}, now we explicitly consider the influence
of uncertainties and disturbances over the closed-loop system. As
a consequence, the Lyapunov derivative yields (see \eqref{eq:derivative_lyapunov_with_disturbance})
\begin{gather}
\begin{aligned}\dot{V}\left(\errorvardq\left(t\right)\right)= & \,\overset{\dot{V}_{1}\left(\errorvardq\left(t\right)\right)}{\overbrace{-\alpha_{1}\inner{\teorientation{\errorvardq}}{\KO\teorientation{\errorvardq}+\noiseaddrotation+\noiseposerotation}}}\\
 & \quad\quad\underset{\dot{V}_{2}\left(\errorvardq\left(t\right)\right)}{\underbrace{-\frac{\alpha_{2}}{2}\inner{\tetranslation{\errorvardq}}{\KT\tetranslation{\errorvardq}-\noiseaddposition-\noiseposeposition}}}.
\end{aligned}
\label{eq:proof:Lyapunov Derivative with V}
\end{gather}

\noindent Defining $V_{\gamma\mathcal{_{O}}}\triangleq\norm{\teorientation{\errorvardq}}^{2}-\perfrotation_{1}^{\negthickspace\negthickspace2}\norm{\noiseaddrotation}^{2}-\perfrotation_{2}^{\negthickspace\negthickspace2}\norm{\noiseposerotation}^{2}$
and $V_{\gamma\mathcal{_{T}}}\triangleq\norm{\tetranslation{\errorvardq}}^{2}-\perftranslation_{1}^{\negthickspace\negthickspace2}\norm{\noiseaddposition}^{2}-\perftranslation_{2}^{\negthickspace\negthickspace2}\norm{\noiseposeposition}^{2}$,
Condition~2 is fulfilled if, for all $t\in[0,\infty)$, the following
inequalities hold
\begin{align}
\dot{V}_{1}\left(\errorvardq\left(t\right)\right)+V_{\gamma\mathcal{_{O}}} & \leq0, & \dot{V}_{2}\left(\errorvardq\left(t\right)\right)+V_{\gamma\mathcal{_{T}}} & \leq0.\label{eq:robustness inequalities}
\end{align}
Indeed, under zero initial conditions (i.e., $V\left(\errorvardq\left(0\right)\right)=0$),
integrating the first inequality in \eqref{eq:robustness inequalities}
results in
\begin{align*}
\int_{0}^{\infty}\negthickspace\negthickspace\negthickspace V_{\gamma\mathcal{_{O}}}dt & \leq{-}\negmedspace\int_{0}^{\infty}\negthickspace\negthickspace\negthickspace\dot{V}_{1}\negthinspace\left(\errorvardq\left(t\right)\right)\negthinspace dt{=}V_{1}\negthinspace\left(\errorvardq\left(0\right)\right){-}\lim_{t\to\infty}V_{1}\negthinspace\left(\errorvardq\left(t\right)\right){\leq}0,
\end{align*}
where the last inequality above holds because $V_{1}\left(\errorvardq\left(0\right)\right)=0$
and $V_{1}\left(\errorvardq\left(t\right)\right)\geq0$, $\forall t$,
which implies the first inequality of Condition~2 in Definition~\ref{def:H_inf}.
The same applies to the second inequality in \eqref{eq:robustness inequalities}.

To satisfy \eqref{eq:robustness inequalities}, we use the definition
of inner product as in Footnote~\ref{fn:inner_product} to rewrite
the first inequality in \eqref{eq:robustness inequalities} as\footnote{Notice that $\mymatrix{\Gamma}^{*}$ is the (quaternion) conjugate
transpose of $\mymatrix{\Gamma}\in\mathbb{H}^{m\times n}$, defined
analogously to its complex matrices counterpart.} 
\begin{equation}
\begin{bmatrix}\teorientation{\errorvardq}\\
\noiseaddrotation\\
\noiseposerotation
\end{bmatrix}^{*}\underset{\mymatrix M}{\negmedspace\underbrace{\begin{bmatrix}-\left(\alpha_{1}\KO{-}1\right) & -\nicefrac{\alpha_{1}}{2} & -\nicefrac{\alpha_{1}}{2}\\
-\nicefrac{\alpha_{1}}{2} & -\perfrotation_{1}^{\negthickspace\negthickspace2} & 0\\
-\nicefrac{\alpha_{1}}{2} & 0 & -\perfrotation_{2}^{\negthickspace\negthickspace2}
\end{bmatrix}}\negthickspace}\begin{bmatrix}\teorientation{\errorvardq}\\
\noiseaddrotation\\
\noiseposerotation
\end{bmatrix}\negthickspace{\leq}0.\label{eq:LMI orientation}
\end{equation}
Since $\mymatrix M\leq0$ implies \eqref{eq:LMI orientation},\footnote{Given a symmetric matrix $\mymatrix M\in\mathbb{R}^{n\times n}$,
if $\myvec u^{T}\mymatrix M\mymatrix u\leq0$, $\forall\myvec u\in\mathbb{R}^{n}$,
then $\myvec{\Gamma}^{*}\mymatrix M\myvec{\Gamma}\leq0$, $\forall\myvec{\Gamma}\in\mathbb{H}^{n}$. } by using Schur complements it is possible to show that $\mymatrix M\leq0$
if and only if
\begin{equation}
\KO\geq\frac{1}{\alpha_{1}}+\frac{\alpha_{1}}{4}\left(\perfrotation_{1}^{\negthickspace\negmedspace-2}+\perfrotation_{2}^{\negthickspace\negmedspace-2}\right).\label{eq:ko}
\end{equation}
We repeat the same procedure for the second inequality in \eqref{eq:robustness inequalities}
to obtain
\begin{equation}
\KT\geq\smash{\frac{2}{\alpha_{2}}}+\smash{\frac{\alpha_{2}}{8}}\left(\perftranslation_{1}^{\negthickspace\negmedspace-2}+\perftranslation_{2}^{\negthickspace\negmedspace-2}\right).\label{eq:kt}
\end{equation}

\noindent \textbf{(Minimum instantaneous control effort)} Since there
exist an infinite number of solutions for $\alpha_{1}$ and $\alpha_{2}$
that satisfy \eqref{eq:ko} and \eqref{eq:kt}, we seek $\alpha_{1\text{opt}}$
and $\alpha_{2\text{opt}}$ that minimize the positive control gains
$\KO$ and $\KT$. By letting $f\left(\alpha_{1}\right)\triangleq\alpha_{1}^{-1}+\left(1/4\right)\alpha_{1}\perfrotation$
and $g\left(\alpha_{2}\right)\triangleq2\alpha_{2}^{-1}+\left(1/8\right)\alpha_{2}\perftranslation$,
where $\perfrotation{\triangleq}\perfrotation_{1}^{\negthickspace\negmedspace-2}+\perfrotation_{2}^{\negthickspace\negmedspace-2}$
and $\perftranslation{\triangleq}\perftranslation_{1}^{\negthickspace\negmedspace-2}+\perftranslation_{2}^{\negthickspace\negmedspace-2}$,
we minimize $f(\alpha_{1})$ and $g(\alpha_{2})$ to obtain $\alpha_{1\text{opt}}=2\perfrotation^{\negthickspace\negmedspace-1/2}$
and $\alpha_{2\text{opt}}{=}4\perftranslation^{\negthickspace\negmedspace-1/2}$.
Thus, the minimum values for the control gains $\KO$ and $\KT$ that
satisfy \eqref{eq:ko} and \eqref{eq:kt} are 
\begin{align*}
\KO & =f\left(\alpha_{1\text{opt}}\right)=\left(\perfrotation_{1}^{\negthickspace\negmedspace-2}+\perfrotation_{2}^{\negthickspace\negmedspace-2}\right)^{1/2},\\
\KT & =g\left(\alpha_{2\text{opt}}\right)=\left(\perftranslation_{1}^{\negthickspace\negmedspace-2}+\perftranslation_{2}^{\negthickspace\negmedspace-2}\right)^{1/2}.
\end{align*}
If $\gamma\triangleq\perftranslation_{1}{=}\perftranslation_{2}{=}\perfrotation_{1}{=}\perfrotation_{2}$
then $\kappa=\KO{=}\KT=\sqrt{2}\gamma^{-1}$. Since the closed-loop
system \eqref{eq:tracking:error_dynamics:jacobian},\eqref{THM:Tracking Controller}
is equivalent to \eqref{eq:equivalent tracking controller}, where
$\dot{\bar{\myvec{\junta}}}{=}\kappa\mymatrix J^{+}\begin{bmatrix}\vectorquat\teorientation{\errorvardq}^{T} & -\vectorquat\tetranslation{\errorvardq}^{T}\end{bmatrix}^{T}$,
then $\norm{\dot{\bar{\myvec{\junta}}}}=\kappa\norm{\myvec{\Gamma}_{\kappa}}$,
with $\myvec{\Gamma}_{\kappa}{=}\mymatrix J^{+}\begin{bmatrix}\vectorquat\teorientation{\errorvardq}^{T} & -\vectorquat\tetranslation{\errorvardq}^{T}\end{bmatrix}^{T}$.
Therefore, since $\kappa$ is the minimum gain that satisfies the
disturbance attenuation specification $\gamma$, then $\norm{\dot{\bar{\myvec{\junta}}}}$
is the minimum instantaneous control effort.
\end{proof}

\allowdisplaybreaks

\section{Simulation Results\label{sec:Experiments}}


To validate and quantitatively assess the performance of the proposed
techniques under different scenarios and conditions, this section
presents simulated results of a KUKA LBR-IV arm connected to a Barrett
Hand. The DQ Robotics toolbox from \citep{Adorno2019} was used for
both robot modeling and control using dual quaternion algebra. Simulations
are performed in V-REP,\footnote{https://www.coppeliarobotics.com/}
in asynchronous mode with $5\,\unit{ms}$ sampling period, with Bullet
2.83 to realistically simulate the robot dynamics.\footnote{http://bulletphysics.org}

\begin{figure}[t]
	\centering \includegraphics[width=0.87\columnwidth]{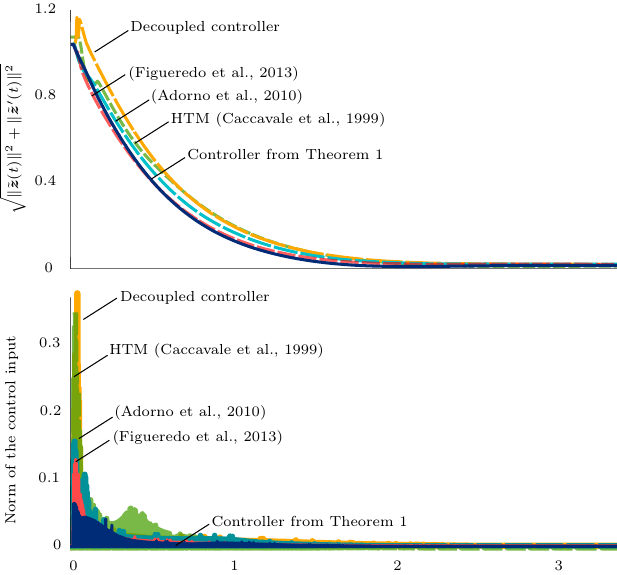}
	
	\caption{Set-point control error (\emph{top}) and control-input norm (\emph{bottom}).}
	\label{fig:ex:setpoint}
\end{figure}

\subsection{Set-point control\label{subsec:Setpoint-control}}

\noindent For the first scenario, the initial manipulator end-effector
pose was $\dq x_{0}=\quat r_{0}+\left(1/2\right)\dual\quat p_{0}\quat r_{0}$,
with $\quat r_{0}{=}\cos\left(\phi_{0}/2\right)+\quat n_{0}\sin\left(\phi_{0}/2\right)$
such that $\phi_{0}=2.187\,\unit{rad}$ and $\quat n_{0}{=}-0.689\imi+0.395\imj+0.606\imk$,
from where it was supposed to travel to $\dq x_{d}{=}\quat r_{d}+\left(1/2\right)\dual\quat p_{d}\quat r_{d}$
with $\quat r_{d}{=}\cos\left(\pi/4\right)+\imj\sin\left(\pi/4\right)$
and $\quat p_{d}{=}1.56\imi-0.43\imj+0.65\imk$.

To evaluate Theorem~\ref{THEOREM_TRACKING} in a regulation problem,
we compared the control law (\ref{THM:Tracking Controller}), with
$\twist_{d}{=}0$, with two different controllers based on dual quaternion
representation \citep{2010_Adorno_IROS,2013_Figueredo_Adorno_Ishihara_Borges_ICRA},
a decoupled controller that concerns independent attitude and translation
task Jacobians (\ref{sec:controllers}), and a classic HTM-based controller
\citep{1999_Caccavale_Siciliano_Villani_AJC}. To allow a fair comparison,
all controllers were set with the same constant control gain $\kappa_{\mathcal{O}}=\kappa_{\mathcal{T}}=\kappa=2$.

The error norm in Fig.~\ref{fig:ex:setpoint} (top figure) shows
similar convergence for all controllers,\footnote{In this section, we used the real end-effector pose from V-REP.}
as expected for undisturbed scenarios, because all of them result
in a similar closed-loop first-order differential equation (in their
own error variables) and they have the same gain. In contrast, the
norm of the control inputs (i.e., the instantaneous control effort),
shown in the bottom figure, indicates that the controller from Theorem~\ref{THEOREM_TRACKING}
requires the least amount of control effort. This is due to the fact
that, although all controllers have the same gain (which ensures the
same convergence rate), they employ different error metrics, hence
resulting in different end-effector trajectories as not all error
metrics respect the topology of the space of rigid motions, which
in turn require different control efforts.


\subsection{Tracking }

\noindent To evaluate Theorem~\ref{THEOREM_TRACKING} in a tracking
problem, the end-effector was prescribed to follow a desired task
trajectory towards the end-pose $\dq x_{d}(t_{f})=\quat r_{d}(t_{f})+(1/2)\dual\quat p_{d}(t_{f})\quat r_{d}(t_{f})$,
where $\quat r_{d}(t_{f})=0.67\imi+0.01\imj-0.74\imk$ and $\quat p_{d}(t_{f})=0.05\imi-1.15\imj+0.75\imk$.
We compared Theorem~\ref{THEOREM_TRACKING} with the same controllers
from the previous case. All controllers were set with control gain
$\kappa=5$.

The trajectory tracking error is shown in Fig.~\ref{fig:ex:tracking : error new}.
The dark blue curve concerns the result based on the tracking control
law of Theorem~\ref{THEOREM_TRACKING}. The result demonstrates the
improved performance when compared to results from \citealt{2010_Adorno_IROS,2013_Figueredo_Adorno_Ishihara_Borges_ICRA},
decoupled controller, and HTM-based controller \citep{1999_Caccavale_Siciliano_Villani_AJC},
all of them with similar control effort, as shown in Fig.~\ref{fig:ex:tracking : control effort},
which highlights the importance of using a proper feedforward correction
term during tracking control.

\begin{figure}
	\centering \includegraphics[width=0.95\columnwidth]{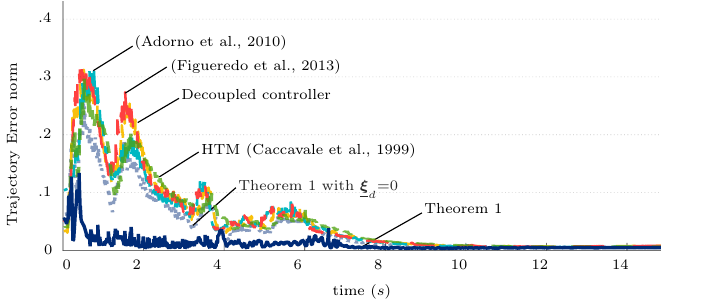}
	
	\caption{Tracking control error.}
	\label{fig:ex:tracking : error new}
\end{figure}

\begin{figure}
	\centering \includegraphics[width=0.95\columnwidth]{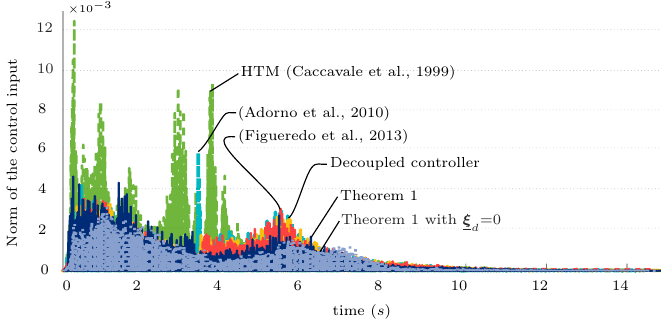}
	
	\caption{Trajectory tracking: norm of the control input.}
	\label{fig:ex:tracking : control effort}
\end{figure}

\subsection{$H_{\infty}$ robustness}

To illustrate the performance of the proposed robust $H_{\infty}$
controller under different uncertainties and disturbances, a task
was devised based on the motion of a mobile platform, a Pioneer P3-DX,
which moved in triangle-wave fashion, alternating smoothly back and
forth at fixed speed (respectively with period of $2.5\,s$ and $3.45\,s$).
The end-effector had to track the non-fixed target with a constant
relative pose. Since in this scenario the robot manipulator does not
have knowledge of the mobile base velocity, the trajectory has an
additional unknown twist, which is a disturbance that directly affects
the relative pose.

Theorem~\ref{THEOREM_TRACKING} was used with different values of
$\perftranslation$, while keeping $\perfrotation=2$ constant. Table~\ref{table:diff hinf values}
summarizes the numerically computed noise to error attenuation, 
\begin{align*}
	\perftranslation_{\text{sim}} & {=}\tfrac{\int_{0}^{T}\norm{\tetranslation{\errorvardq\left(t\right)}}^{2}dt}{\int_{0}^{T}\norm{\noiseaddposition\left(t\right)}^{2}+\norm{\noiseposeposition\left(t\right)}^{2}dt}, & \negthickspace\perfrotation_{\text{sim}}{=} & \tfrac{\int_{0}^{T}\norm{\teorientation{\errorvardq\left(t\right)}}^{2}dt}{\int_{0}^{T}\norm{\noiseaddrotation\left(t\right)}^{2}+\norm{\noiseposerotation\left(t\right)}^{2}dt}.
\end{align*}
As expected from the $H_{\infty}$ norm given by Definition~\ref{def:H_inf},
the noise to error attenuation remains below the prescribed thresholds,
i.e., $\perfrotation_{\text{sim}}{\leq}\perfrotation$ and $\perftranslation_{\text{sim}}{\leq}\perftranslation$,
for all $\perfrotation$, $\perftranslation$.

\begin{table}[t]
	\caption{Comparison between theoretical upper bound ($\protect\perftranslation$,
		$\protect\perfrotation$) with the numerically calculated noise-to-error
		attenuation ($\protect\perftranslation_{\text{sim}}$, $\protect\perfrotation_{\text{sim}}$).}
	\vspace{-5pt}\setlength{\tabcolsep}{4pt}
	\begin{centering}
		{%
			\begin{tabular}{cccccccc}
				\hline 
				$\negthickspace$$\negthickspace$$\perftranslation$ & $3.5$ & $2.0$ & $0.9$ & $0.6$ & $0.5$ & $0.4$ & $0.2$\tabularnewline
				$\perftranslation_{\text{sim}}$ & $1.914$ & $1.235$ & $0.736$ & $0.528$ & $0.404$ & $0.326$ & $0.167$\tabularnewline
				\hline 
				$\negthickspace$$\negthickspace$$\perfrotation$ & $2$ & $2$ & $2$ & $2$ & $2$ & $2$ & $2$\tabularnewline
				$\perfrotation_{\text{sim}}$ & $0.95$ & $1.01$ & $0.99$ & $1.03$ & $1.04$ & $1.01$ & $0.99$\tabularnewline
				\hline 
		\end{tabular}}
		\par\end{centering}
	\label{table:diff hinf values}
\end{table}

The proposed controller, with $\perftranslation{=}0.4,\,\perfrotation{=}1$,
was again compared to the dual-quaternion based controllers from \citet{2010_Adorno_IROS}
and \citet{2013_Figueredo_Adorno_Ishihara_Borges_ICRA}, and the HTM-based
controller \citep{1999_Caccavale_Siciliano_Villani_AJC}. To maintain
fairness, all controllers were manually set to ensure similar control
effort in terms of $\int_{0}^{T}\norm{\myvec u\left(t\right)}dt$.
The numerically calculated noise-to-error attenuation from the simulations,
presented in Table~\ref{table:hinf comparison}, shows that for the
same control effort our controller outperforms the other ones in terms
of disturbance attenuation.

\begin{table}[t]
	\caption{Numerically computed noise-to-error attenuation ($\protect\perftranslation_{\text{sim}}$,
		$\protect\perfrotation_{\text{sim}}$) compared against theoretical
		values $\protect\perftranslation{=}0.4$ and $\protect\perfrotation{=}1.0$.}
	
	\vspace{-5pt}\setlength{\tabcolsep}{3pt}
	\begin{centering}
		\begin{tabular}{cccc}
			\hline 
			Theorem~\ref{THEOREM_TRACKING} & {\small \citeauthor{2014_Figueredo_Adorno_Ishihara_Borges__IROS}} & {\small \citeauthor{2010_Adorno_IROS}} & {\small HTM}\tabularnewline
			$\perftranslation_{\text{sim}}{=}\boldsymbol{0.32}$ & $\perftranslation_{\text{sim}}{=}0.63$ & $\perftranslation_{\text{sim}}{=}0.61$ & $\perftranslation_{\text{sim}}{=}0.67$\tabularnewline
			$\perfrotation_{\text{sim}}{=}\boldsymbol{0.65}$ & $\perfrotation_{\text{sim}}{=}0.86$ & $\perfrotation_{\text{sim}}{=}0.87$ & $\perfrotation_{\text{sim}}{=}0.78$\tabularnewline
			\hline 
		\end{tabular}
		\par\end{centering}
	\label{table:hinf comparison}\vspace{-10pt}
\end{table}


\section{Experimental Results \label{sec:realExperiments}}


This section presents results from the implementation on a real Meka
Robotics A2 Arm, which is a highly compliant anthropomorphic $7$-joint
manipulator and presents several unmodeled dynamic effects. We defined
a trajectory over a helix curve in space with $10$ cm of both radius
and axis length. Fig.~\ref{fig:mekaRobot} illustrates the trajectory
with the aid of light-painting technique.\footnote{A photographic technique of moving a light source while taking a long
	exposure photograph, which leaves a trail in the final image.} We implemented the controller using C$++$ and the DQ Robotics toolbox
\citep{Adorno2019} on ROS\footnote{\url{https://www.ros.org/}} with
a $8$ ms sampling period. For the experiment, the end-effector pose
used in the control-loop was computed through the FKM.

The task trajectory was executed with different values for $\perfvariable=\perftranslation=\perfrotation$.
Each experimental condition was then executed ten times for statistical
significance (a total of $50$ trials). Results are summarized in
Table~\ref{table:real experiment} in terms of mean square error
and standard deviation integrated along the trajectory. As expected,
the $H_{\infty}$ controller is able to deal with disturbances originated,
for instance, from the coupled nonlinear dynamics, measurement noise,
parameter's uncertainties, among others. Indeed, assuming similar
disturbances conditions in all trials\textemdash which is reasonable
as the experimental conditions were the same\textemdash and normalizing
the results over the worst performance (i.e., the mean square error
$\mathrm{MSE}_{\max}$ corresponding to $\perfvariable{=}\bar{\gamma}{\triangleq}1.86$),
the average disturbance-to-error attenuation improvements ($\nicefrac{\mathrm{MSE}_{\max}}{\mathrm{MSE}}$)
were almost inversely proportional to the decrease in $\perfvariable$.
Fig.~\ref{fig:Error-trajectory-for} shows the error along the trajectory
for one execution of the robust controller for performance bounds
$\perfvariable\in\{1.86,0.93,0.37\}$. As the theory predicts, setting
smaller values for $\perfvariable$ yields a controlled system with
better disturbance attenuation, which is manifested by smaller errors.
In this case, since we did not directly measure the end-effector pose,
the error is given with respect to the nominal value obtained through
the FKM.
\begin{figure}
	\noindent \begin{centering}
		\subfloat[Meka Robotics A2 Arm (\emph{left}) and executed helicoidal trajectory
		with light-painting (\emph{right}).]{\noindent \centering{}%
			\begin{tabular}{cc}
				\includegraphics[width=0.45\columnwidth]{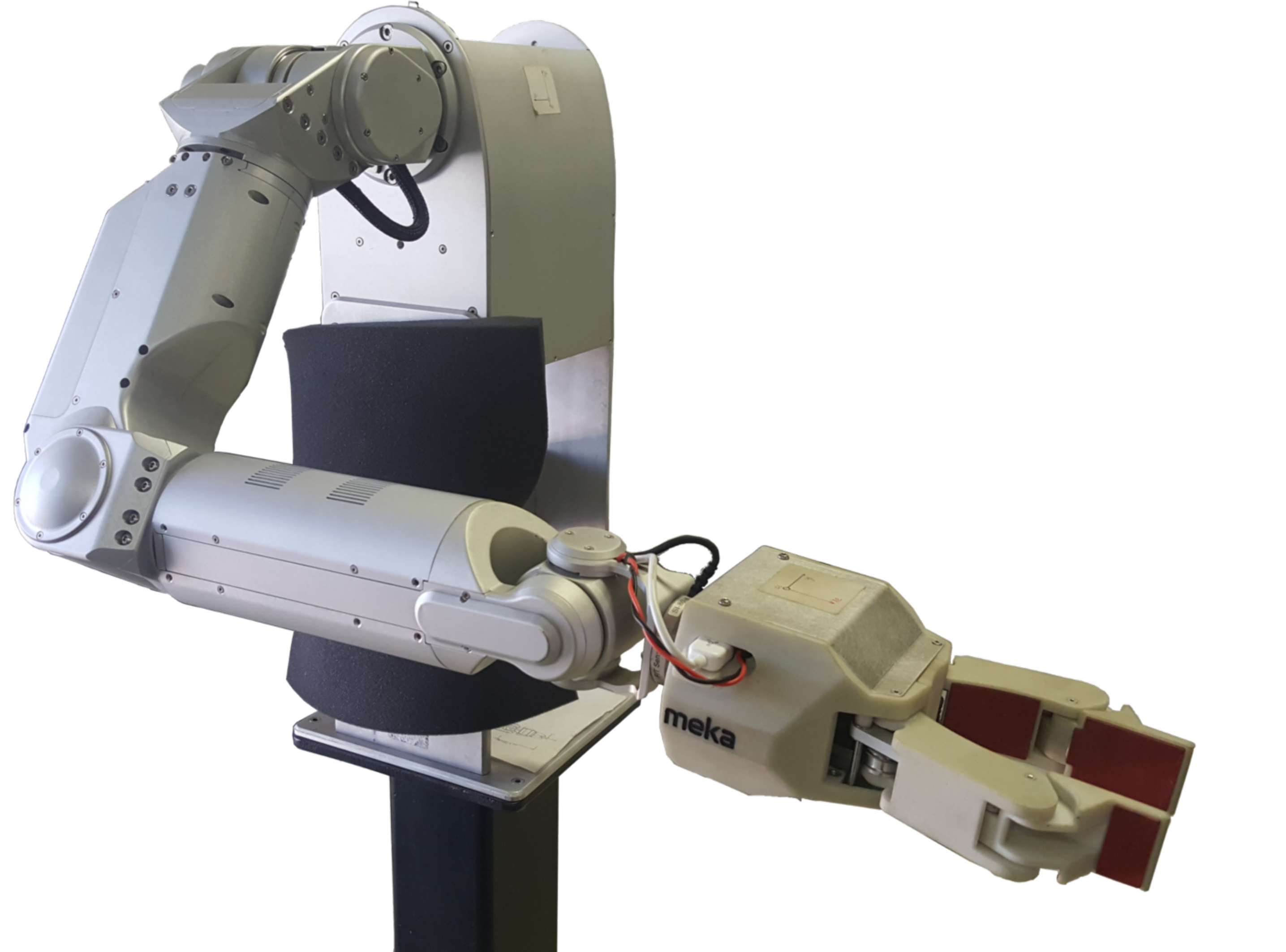} & \includegraphics[width=0.4\columnwidth]{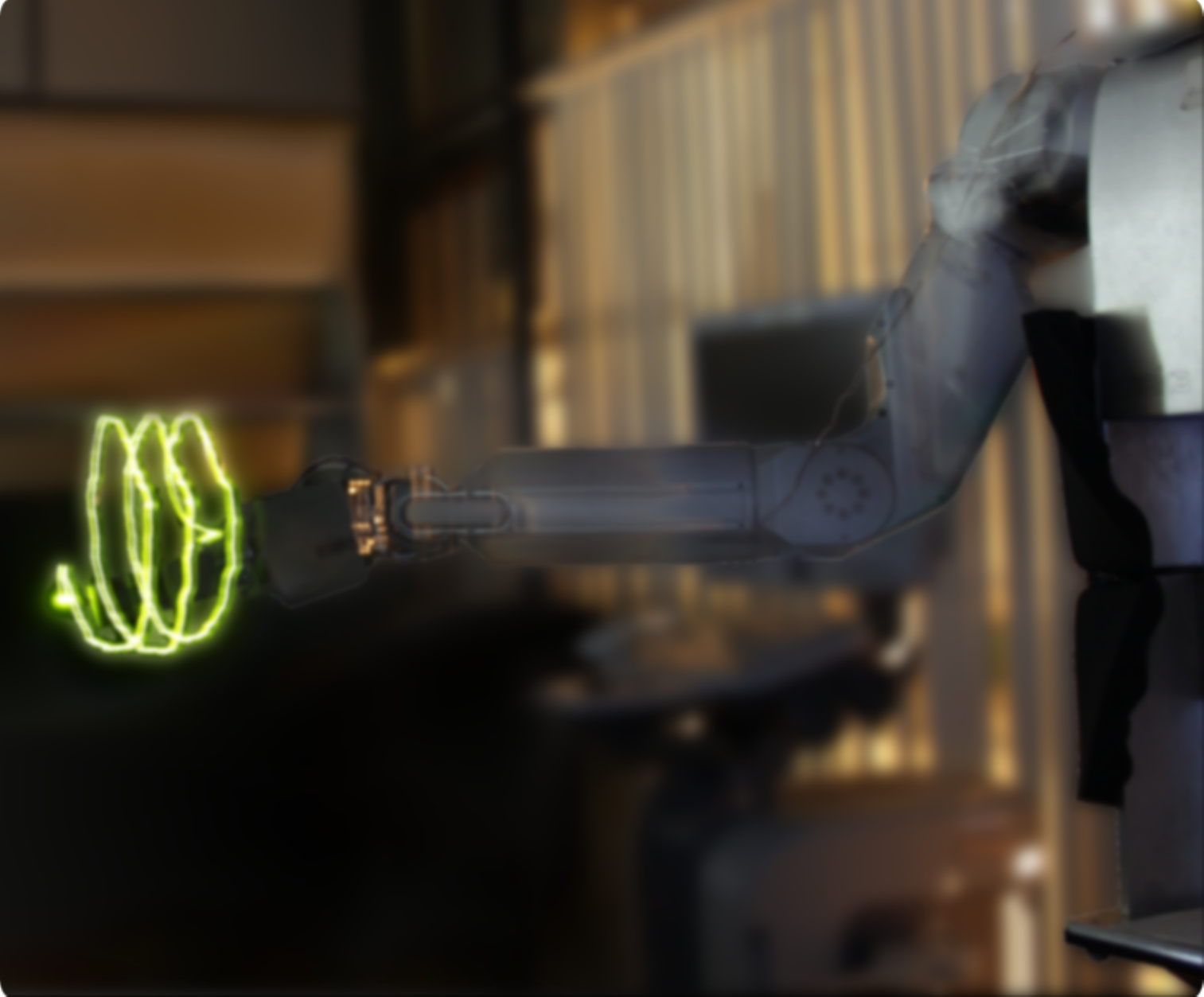}\tabularnewline
			\end{tabular}\label{fig:mekaRobot:MekaPhoto}}
		\par\end{centering}
	\noindent \centering{}\subfloat[Error trajectory for different values of prescribed noise-to-error
	upper bound $\gamma=\gamma_{\mathcal{T}}=\gamma_{\mathcal{O}}$ for
	one execution.\label{fig:Error-trajectory-for}]{\centering \includegraphics[width=0.95\columnwidth]{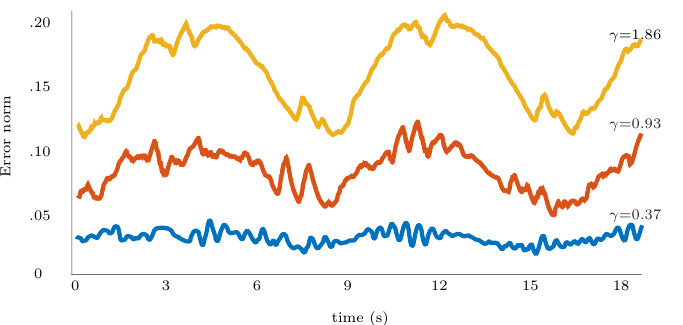}}\caption{Experiments on the real robot manipulator.}
	\label{fig:mekaRobot}
\end{figure}
\textcolor{blue}{}
\begin{table*}
	\caption{Average mean square error (MSE) and standard deviation (STD) under
		different performance conditions for the real experiment.\textcolor{blue}{\label{table:real experiment}}}
	\vspace{-5pt}
	\centering{}%
	\begin{tabular*}{1\textwidth}{@{\extracolsep{\fill}}cccccc}
		\hline 
		& \vspace{5pt}$\bar{\gamma}{\triangleq}1.86$ & $\frac{1}{2}\bar{\gamma}=0.93$ & $\frac{1}{3}\bar{\gamma}=0.62$ & $\frac{1}{4}\bar{\gamma}=0.47$ & $\frac{1}{5}\bar{\gamma}=0.37$\tabularnewline
		\hline 
		$\mathrm{MSE}\pm\mathrm{STD}$ & $\boldsymbol{3.118}\pm0.015$ & $1.662\pm0.007$ & $1.123\pm0.011$ & $0.829\pm0.011$ & $0.662\pm0.005$\tabularnewline
		$\nicefrac{\mathrm{MSE}_{\max}}{\mathrm{MSE}}$ & $1$ & $1.87$ & $2.78$ & $3.76$ & $4.71$\tabularnewline
		\hline 
	\end{tabular*}
\end{table*}

\subsection{Practical Considerations}

When implementing the controller (\ref{THM:Tracking Controller})
on a digital computer, instability issues may arise if the tracking
bandwidth is too high compared to the inner joint-level control-loop
bandwidth. Fortunately, modern manipulators have very fast joint controllers,
often around 1 KHz, while the outer-loop often runs between 20\textendash 125
Hz. Also, considering the discrete-time joint dynamics given by $\Delta\myvec q_{k+1}{=}a\Delta\myvec q_{k}{+}bT\myvec u$,
where $a\in(-1,1)$, $b>0$, $\Delta\myvec q_{k}=\myvec q_{k}-\myvec q_{k-1}$,
and $T$ is the sampling time, \citet{Bjerkeng2014} show that the
gain $\kappa=\kappa_{\mathcal{T}}=\mathcal{\kappa_{\mathcal{O}}}$
for a controller such as (\ref{THM:Tracking Controller}) should respect
$\kappa<2(1+a)/bT$. Assuming $T=1\unit{ms}$ and a tuned controller
(i.e., $a\approx0$), one has $\kappa<2000/b$. Thus, $\kappa$ can
still be very large without affecting the practical closed-loop stability.

\section{Further Discussions and Conclusions}

\label{sec:Conclusion}

In this paper, we have exploited the geometrical significance of the
dual quaternion algebra to derive an easy-to-implement closed-form
$H_{\infty}$ task-space controller for the non-Euclidean space $\text{Spin}(3)\ltimes\mathbb{R}^{3}$
that describes the end-effector pose. Realistic simulations and experiments
on a real robot were performed in different conditions and with different
control strategies, which led to the following conclusions: \emph{a})
compared to similar controllers with same convergence rate for regulation,
the proposed controller requires less instantaneous control effort
when no disturbances affect the system, and it has improved tracking
performance; \emph{b}) when there are disturbances, if all controllers
are tuned to have similar control effort, the proposed controller
ensures less set-point and tracking errors.

Lastly, our method works even when the end-effector pose is not directly
measured because it can be estimated using the FKM. Although for most
commercial robots the FKM provides sufficiently accurate end-effector
poses, the nominal value may still differ from the real one due to
several reasons. In that case, our strategy may be used to attenuate
the influence of any disturbances over the end-effector trajectory,
even if the FKM is not accurate enough. The residual error will then
be bounded by the magnitude of the unknown transformation between
the estimated and the actual end-effector pose.


\appendix

\section{Derivative of the Lyapunov function\label{APPENDIX_A}}


From \eqref{eq:error:NonDisturbed:new metrics}, $\errorvardq\triangleq1-\error{\dq x}=\errorvarquat+\dual\errorvarquat'$.
By letting $\error{\dq x}\triangleq\eta+\quat{\mu}+\dual\left(\eta'+\quat{\mu}'\right)$,
the positive definite functions $V_{1}$ and $V_{2}$ in the Lyapunov
function \eqref{THM:LYAPUNOV} can be rewritten as
\begin{align*}
V_{1}(\errorvarquat\left(t\right)) & =\alpha_{1}\norm{\errorvarquat(t)}^{2}{=}\alpha_{1}\left(\left(1{-}\eta\right)^{2}+\norm{\quat{\mu}}^{2}\right){=}2\alpha_{1}\left(1{-}\eta\right),\\
V_{2}(\errorvarquat'\left(t\right)) & =\alpha_{2}\norm{\errorvarquat'(t)}^{2}{=}\alpha_{2}\left(\eta'^{2}+\norm{\quat{\mu}'}^{2}\right).
\end{align*}
The derivative of \eqref{THM:LYAPUNOV} yields $\dot{V}_{1}(\errorvarquat\left(t\right))+\dot{V}_{2}(\errorvarquat'\left(t\right))$
with
\begin{align*}
\dot{V}_{1}(\errorvarquat\left(t\right)) & {=}-2\alpha_{1}\dot{\eta},\,\,\text{and }\dot{V}_{2}(\errorvarquat'\left(t\right)){=}2\alpha_{2}\eta'\dot{\eta}'+2\alpha_{2}\inner{\quat{\mu}'}{\dot{\quat{\mu}}'}.
\end{align*}
Using the closed-loop dynamics \eqref{eq:equivalent tracking controller},\footnote{Those hold even if $\mymatrix J$ is not full row rank ($\mymatrix J\mymatrix J^{+}{\neq}\mymatrix I$),
which usually happens in a singular configuration. In that case, let
$\myvec s\triangleq\begin{bmatrix}\KO\left(\vectorquat\teorientation{\errorvardq}\right)^{T} & -\KT\left(\vectorquat\tetranslation{\errorvardq}\right)^{T}\end{bmatrix}^{T}$
then $\vectorinv\left(\mymatrix J\mymatrix J^{+}\myvec s\right)=\vectorinv\left(\myvec s\right)+\dq v_{s}$,
where $\dq v_{s}$ is a disturbance to be added into $\noiseadd$.} we obtain
\begin{align*}
\dot{\eta} & =-\tfrac{1}{2}\inner{\quat h_{1}}{\quat{\mu}},\quad\quad\dot{\eta}'=-\tfrac{1}{2}\left(\inner{\quat h_{1}}{\quat{\mu}'}+\inner{\quat h_{2}}{\quat{\mu}}\right),\\
\dot{\quat{\mu}}' & =\tfrac{1}{2}\left(\eta'\quat h_{1}+\eta\quat h_{2}+\quat h_{1}\times\quat{\mu}'+\quat h_{2}\times\quat{\mu}\right),
\end{align*}
where\footnote{Herein, we use the fact that given $\quat u,\quat v\in\quatset_{p}$,
$\quat u\quat v=-\dotproduct{\quat u,\quat v}+\quat u\times\quat v$,
where both cross product, $\quat u\times\quat v\triangleq(\quat u\quat v-\quat v\quat u)/2$,
and inner product, $\inner{\quat u}{\quat v}\triangleq-\left(\quat u\quat v+\quat v\quat u\right)/2$,
are equivalent to their counterparts in $\mathbb{R}^{3}$.\label{fn:inner_product}} $\quat h_{1}=\KO\teorientation{\errorvardq}+\noiseaddrotation+\noiseposerotation$
and $\quat h_{2}=-\KT\tetranslation{\errorvardq}+\noiseaddposition+\noiseposeposition.$
Hence,
\begin{align}
\dot{V}_{1}(\errorvarquat(t)) & =\alpha_{1}\inner{\quat{\mu}}{\quat h_{1}},\label{eq:Lyapunov derivative (V1)}\\
\dot{V}_{2}(\errorvarquat'(t)) & =\alpha_{2}\inner{\eta\quat{\mu}'-\eta'\quat{\mu}+\quat{\mu}\times\quat{\mu}'}{\quat h_{2}}.\label{eq:Lyapunov derivative (V2)}
\end{align}
To investigate the first condition from Definition~\ref{def:H_inf},
which regards exponential stability of \eqref{eq:equivalent tracking controller}
in the absence of disturbances $\noiseadd$ and $\noisepose$, let
us rewrite \eqref{eq:Lyapunov derivative (V1)}-\eqref{eq:Lyapunov derivative (V2)}
as $\dot{V}(\errorvardq\left(t\right))=\dot{V}_{1}(\errorvarquat\left(t\right))+\dot{V}_{2}(\errorvarquat'\left(t\right))$
with
\begin{align*}
\dot{V}_{1}(\errorvarquat\left(t\right)) & =\alpha_{1}\inner{\myvec{\mu}}{\KO\teorientation{\errorvardq}},\\
\dot{V}_{2}(\errorvarquat'\left(t\right)) & =-\alpha_{2}\inner{\eta\myvec{\mu}'-\eta'\myvec{\mu}+\myvec{\mu}\times\myvec{\mu}'}{\KT\tetranslation{\errorvardq}}.
\end{align*}
From \eqref{eq:error_at_origin} and considering the unit dual quaternion
constraint $\eta\eta'+\inner{\quat{\mu}}{\quat{\mu}'}=0$ \citep{Kussaba2017},
we have $\tetranslation{\errorvardq}=2(\eta\myvec{\mu}'-\eta'\myvec{\mu}+\myvec{\mu}\times\myvec{\mu}')$
and $\teorientation{\errorvardq}=-\myvec{\mu}$; therefore,
\begin{align*}
\dot{V}_{1}(\errorvarquat\left(t\right)) & =-\alpha_{1}\KO\inner{\myvec{\mu}}{\myvec{\mu}}=-\alpha_{1}\KO\norm{\myvec{\mu}}^{2},\\
\begin{split}\dot{V}_{2}(\errorvarquat'\left(t\right)) & =-2\alpha_{2}\KT\left(\eta^{2}\inner{\myvec{\mu}'}{\myvec{\mu}'}+\eta'^{2}\inner{\myvec{\mu}}{\myvec{\mu}}-2\eta\eta'\inner{\myvec{\mu}}{\myvec{\mu}'}\right.\\
 & \left.\hphantom{=-2\alpha_{2}\KT+}+\inner{\myvec{\mu}}{\myvec{\mu}}\inner{\myvec{\mu}'}{\myvec{\mu}'}-\inner{\myvec{\mu}}{\myvec{\mu}'}^{2}\right).
\end{split}
\end{align*}
Notice the identity $\inner{\myvec{\mu}{\times}\myvec{\mu}'}{\myvec{\mu}{\times}\myvec{\mu}'}{=}\inner{\myvec{\mu}}{\myvec{\mu}}\inner{\myvec{\mu}'}{\myvec{\mu}'}{-}\inner{\myvec{\mu}}{\myvec{\mu}'}^{2}$
in the last equality. Since $\inner{\myvec{\mu}}{\myvec{\mu}'}=-\eta\eta'$
and $\eta^{2}+\norm{\quat{\mu}}^{2}=1$,\begin{small}
\begin{align}
\dot{V}_{1}(\errorvarquat\left(t\right)) & {=}-\alpha_{1}\frac{\KO}{2}\left(\norm{\myvec{\mu}}^{2}{+}\norm{\myvec{\mu}}^{2}\right)=-\alpha_{1}\frac{\KO}{2}\left(1{-}\eta^{2}{+}\norm{\myvec{\mu}}^{2}\right),\nonumber \\
\dot{V}_{2}(\errorvarquat'\left(t\right)) & {=}{-}2\alpha_{2}\KT\Bigl(\negmedspace\eta^{2}\negmedspace\norm{\myvec{\mu}'}^{2}{+}\eta'^{2}\norm{\myvec{\mu}}^{2}{+}\eta^{2}\eta'^{2}{+}\left(1{-}\eta^{2}\right)\negmedspace\norm{\myvec{\mu}'}^{2}\negmedspace\Bigr)\nonumber \\
 & {=}-2\alpha_{2}\KT\left(\eta'^{2}{+}\norm{\myvec{\mu}'}^{2}\right){=}{-}2\alpha_{2}\KT\norm{\errorvarquat'(t)}^{2}\negmedspace.\negthickspace\label{eq:Derivative of V2}
\end{align}
\end{small}For $\eta{\in}[0,1],$ it is easy to see that $\left(1{-}\eta\right)^{2}\leq\left(1{-}\eta^{2}\right).$
Thus,
\begin{equation}
\dot{V}_{1}(\errorvarquat\left(t\right))\leq-\alpha_{1}\frac{\KO}{2}\left(\left(1-\eta\right)^{2}+\norm{\myvec{\mu}}^{2}\right)=-\alpha_{1}\frac{\KO}{2}\norm{\errorvarquat(t)}^{2},\label{eq:Derivative of V1 (inequality)}
\end{equation}
and, therefore, 
\begin{align*}
\dot{V}(\errorvardq\left(t\right)) & \leq-\alpha_{1}\frac{\KO}{2}\norm{\errorvarquat(t)}^{2}-2\alpha_{2}\KT\norm{\errorvarquat'(t)}^{2},
\end{align*}
which in turn yields \eqref{eq:proof:Lyapunov Derivative 1}.
\begin{rem}
\label{rem:addressing_the_unwinding}To address the interval $\eta\in[-1,0]$
and prevent the problem of unwinding \citep{Kussaba2017}, one must
assume $\errorvardq=1+\tilde{\dq x}$ instead of \eqref{eq:error:NonDisturbed:new metrics}.
Hence, without loss of generality, the exact same controller from
Theorem~\ref{THEOREM_TRACKING} yields \eqref{eq:proof:Lyapunov Derivative 1}
with $\norm{\errorvarquat(t)}^{2}=\left(1{+}\eta\right)^{2}{+}\norm{\myvec{\mu}}^{2}$
where $\eta=-1$ is the equilibrium.\footnote{One must only observe that $\teorientation{\errorvardq}=\myvec{\mu}$
when $\errorvardq=1+\tilde{\dq x}$, and the inequality $\left(1+\eta\right)^{2}\leq1-\eta^{2}$
holds when $\eta\in\left[-1,0\right]$ .}
\end{rem}
Now, if we explicitly regard the influence of $\noiseadd$ and $\noisepose$,
the Lyapunov derivative \eqref{eq:Lyapunov derivative (V1)}-\eqref{eq:Lyapunov derivative (V2)}
yields\footnote{Recall that $\tetranslation{\errorvardq}=2\left(\eta\myvec{\mu}'-\eta'\myvec{\mu}+\myvec{\mu}\times\myvec{\mu}'\right)$
and $\teorientation{\errorvardq}=-\quat{\mu}$.}
\begin{align}
\dot{V}(\errorvardq\left(t\right)) & =-\alpha_{1}\inner{\teorientation{\errorvardq}}{\KO\teorientation{\errorvardq}{+}\noiseaddrotation{+}\noiseposerotation}\nonumber \\
 & \qquad+\frac{\alpha_{2}}{2}\inner{\tetranslation{\errorvardq}}{-\KT\tetranslation{\errorvardq}+\noiseaddposition{+}\noiseposeposition},\label{eq:derivative_lyapunov_with_disturbance}
\end{align}
which is equivalent to \eqref{eq:proof:Lyapunov Derivative with V}.

\section{Decoupled controller\label{sec:controllers}}

Given $\dq x=\quat r+\left(1/2\right)\dual\quat p\quat r$ and $\dq x_{d}=\quat r_{d}+\left(1/2\right)\dual\quat p_{d}\quat r_{d}$,
the control input is 
\begin{align*}
\dot{\quat q} & \negmedspace=\negmedspace\mymatrix J_{\mathrm{dec}}^{+}\kappa\negmedspace\left[\negmedspace\begin{array}{c}
\vectorquat\left(\quat p_{d}-\quat p\right)\\
\mathrm{vec}_{4}\left(1-\quat r^{*}\quat r_{d}\right)
\end{array}\negmedspace\right]\negmedspace, & \text{with }\mymatrix J_{\mathrm{dec}} & \negmedspace=\negmedspace\begin{bmatrix}\mymatrix J_{p}\\
\mymatrix N_{R_{4}}
\end{bmatrix},
\end{align*}
where $\mathrm{vec}_{4}\,:\,\mathcal{\mathbb{H}}\to\mathbb{R}^{4}$
is analogous to $\vectorquat$ and the velocity satisfies $\vectorquat\dot{\quat p}=\mymatrix J_{p}\dot{\myvec q}$.
The matrix $\mymatrix N_{R_{4}}$ corresponds to the four upper rows
of $\mymatrix N_{\!R_{8}}$, in which $\mymatrix N_{\!R_{8}}=\smash{\hamilton -{\dq x_{d}}}\mymatrix C_{8}\mymatrix J_{\!R_{8}}$,
with $\mymatrix C_{8}=\mathrm{diag}\left(1,-1,-1,-1,1,-1,-1,-1\right)$,
and $\mymatrix J_{\!R_{8}}$ is the Jacobian matrix that satisfies
$\mathrm{vec}_{8}\dot{\dq x}=\mymatrix J_{\!R_{8}}\dot{\myvec q}$,
where $\mathrm{vec}_{8}\,:\,\mathcal{H}\to\mathbb{R}^{8}$ is analogous
to $\vector$.

\bibliographystyle{elsarticle-harv}
\bibliography{root}

\end{document}